\documentclass[submission, PhysLectNotes]{SciPost}

\hypersetup{
    colorlinks,
    linkcolor={red!50!black},
    citecolor={blue!50!black},
    urlcolor={blue!80!black}
}

\usepackage[bitstream-charter]{mathdesign}
\DeclareSymbolFont{usualmathcal}{OMS}{cmsy}{m}{n}
\DeclareSymbolFontAlphabet{\mathcal}{usualmathcal}

\usepackage[utf8]{inputenc} 
\usepackage[T1]{fontenc}    
\usepackage{amsmath}
\usepackage{amsthm}
 \usepackage{thmtools}

\newcommand{\R}{\mathbb{R}}
\newcommand{\E}{\mathbb{E}}
\newcommand{\N}{\mathbb{N}}

\newcommand{\Id}{\text{Id}_d}

\newcommand{\fwd}{\text{\sc F}}
\newcommand{\rev}{\text{\sc R}}

\makeatletter
\def\paragraph{\@startsection{paragraph}{4}%
  \z@\z@{-\fontdimen2\font}%
  {\normalfont\bfseries}}
\makeatother

\newtheorem{theorem}{Theorem}

\newtheorem{proposition}[theorem]{Proposition}

\newtheorem{remark}{Remark}

\newtheorem{algorithma}{Algorithm}

\begin{document}

\begin{center}{\Large \textbf{
Learning to Sample Better
}}\end{center}

\begin{center}
Michael Albergo\textsuperscript{1} and
Eric Vanden-Eijnden\textsuperscript{2}
\end{center}

\begin{center}
{\bf 1} Center for Cosmology and Particle Physics, New York University
\\
{\bf 2} Courant Institute of Mathematical Sciences, New York University

{\small \sf albergo@nyu.edu};  {\small \sf eve2@cims.nyu.edu}
\end{center}



\section*{Abstract}
{\bf
These lecture notes provide an introduction to recent advances in generative modeling methods based on the dynamical transportation of measures, by means of which samples from a simple base measure are mapped to samples from a target measure of interest. Special emphasis is put on the applications of these methods to Monte-Carlo (MC) sampling techniques, such as importance sampling  and Markov Chain Monte-Carlo (MCMC) schemes.  In this context, it is shown how the maps can be learned variationally using data generated by MC sampling, and how they can in turn be used to improve such sampling in a positive feedback loop.  
}

\vspace{10pt}
\noindent\rule{\textwidth}{1pt}
\tableofcontents\thispagestyle{fancy}


\section{Introduction}
\label{sec:intro}
The calculation of expectations and probabilities is a prevalent objective in almost all areas of sciences. The problem can be generically phrased as follows: Given a probability measure $\mu_*$ defined on some sample space $\Omega\subseteq \R^d$ and an observable  $f: \Omega \to \R$, we wish to estimate
\begin{equation}
    \label{eq:def:expect}
    \mu_*(f) := \int_\Omega f(x) d\mu_*(x).
\end{equation}
Choosing $f$ to be the indicator function of some set $A\subset\Omega$, this gives the probability of $A$. Oftentimes the probability measure $\mu_*$ is known only up to a normalization factor, and the estimation of this factor is also desirable: it is known as the partition function in statistical mechanics and the evidence in Bayesian inference. 

The analytical evaluation of~\eqref{eq:def:expect} is possible only in few very specific cases. Standard numerical quadrature methods based on gridding space are also inapplicable as soon as the set $\Omega$ is high dimensional (i.e. $d$ is larger than 3). In practice one must therefore resort to Monte Carlo (MC) methods whereby one estimates~\eqref{eq:def:expect} by replacing it by some empirical average over samples. Since direct sampling of $\mu_*$ is typically not an option either, the most common Monte Carlo strategies  are importance sampling, in which independent samples drawn from a simpler distribution than $\mu_*$ are properly reweighted \cite{Kahn1950a, Kahn1950b} to estimate~\eqref{eq:def:expect}, and Markov Chain Monte Carlo (MCMC), in which a Markov sequence ergodic with respect to $\mu_*$ is used \cite{Rosenbluth1955} ---these two classes of methods will be discussed in more details below. To improve their efficiency, these vanilla MC strategies are typically integrated into more sophisticated methods such as umbrella sampling (aka stratification)  or replica exchange \cite{Swendsen1987, geyer1995}. As a rule, however, to be effective all these approaches must be carefully tailored to the problem at hand using some prior information about~$\mu_*$ and~$f$. As a result the success of Monte Carlo strategies often relies on the skill of the user. The aim of this paper is to discuss how to leverage recent advances in machine learning (ML) to improve this situation and streamline the design of efficient Monte Carlo sampling strategies. 

This is a natural aim considering that machine learning tools and concepts have revolutionized the way we process large data sets, in particular in situations where these data were drawn from complex high-dimensional distributions. This feat is \textit{a~priori} appealing in the context of Monte Carlo sampling where one is confronted with probability distributions with similar features. There is however one important difference. By and large the successes of ML have relied on the prior availability of large quantities of data. In the context of MC, there is typically no data accessible beforehand: in fact, the whole aim of MC is to generate such data in order to make controlled estimation of \eqref{eq:def:expect} possible. This leads to specific challenges, uncommon in the context of supervised learning, but also opportunities to understand better certain aspects of ML training procedures by relying e.g. on online learning in the context of which issues like convergence or  generalization error can be better understood. Our aim is also to discuss these aspects.

Here, we will discuss several schemes that are based on constructing transport maps in order to improve MC sampling: as a rule, these maps push samples from a simple base distribution onto samples that are better adapted to the target distribution. This methodology is applicable in the context of importance sampling as well as MCMC, and it can also be generalized to non-equilibrium sampling. This leads to a feedback loop in which we sample an objective by a MC method to learn a better transport map, and use this map to produce better samples via MC, much in the spirit of reinforcement learning strategies. While the task of devising maps between distributions is foundational in the field of optimal transport \cite{villani2009optimal, santambrogio2015optimal}, the  variational framing of this problem in machine learning has its roots in \cite{chen2000, tabak2010, tabak2013}.
Further progress in this direction have been made in \cite{noe2019, albergo2019, kanwar2020, Nicoli2020, gabrie2022} where the maps used for sampling are parameterized by deep neural networks (DNN), an innovation realized in \cite{dinh2014,dinh2017density,rezende2015}. The possibility to calculate such maps rests on the approximation power of the  DNN used to parametrize them, as well as our ability to optimize the parameters in these networks. 

This paper will offer little insight about the first assumption: we will simply assume that DNN are sufficiently expressive to the task at hand, basing this belief on the growing empirical evidence that this is indeed the case and that the ongoing design of DNN of increasing complexity will keep improving their capabilities. Regarding the question of optimization, we will carefully design objective functions for the maps that are amenable to  empirical estimation via sampling. This guarantees that training of the DNN used to approximate these maps can be performed in practice using online learning with stochastic gradient descent (SGD). The objective functions we introduce have the feature that all their local minimizers are global; however, this nice feature may not be preserved after nonlinear parametrization by a neural network, and it does not necessarily guarantee fast convergence of SGD nor does it easily allow us to obtain convergence rates. Here too, we will leave this question open to further investigations.

\subsection{Organization} This paper is organized as follows.  In Sec.~\ref{sec:transport} we recall some well-know properties of continuity equations and their solution via the method of characteristics. This method involve integrating ordinary differential equations (ODEs) whose solutions give dynamical maps transporting samples from one measure to another. This exposition will balance Eulerian statements at the level of the probability densities with their corresponding Lagrangian interpretations at the level of individual samples. In Sec.~\ref{sec:learn:transp} we show how the velocity field involved in these dynamical maps can be estimated variationally using various objective functions.  In Sec.~\ref{sec:importance} we  discuss how to use these results in the context of importance sampling to control the variance of the weights in the method.  In Sec.~\ref{sec:mcmc} we revisit the same question in the context of Metropolis-Hastings MCMC schemes, illustrating how the convergence of these schemes can be accelerated by assisting them with transport maps. We end with some concluding remarks in Sec.~\ref{sec:conclu}.

\subsection{Notations and Assumptions} To leave technicalities at a minimum, we will assume that the sample space is  $\Omega = \R^d$. We also assume that the measure $\mu_*$ is absolutely continuous with respect to the Lebesgue measure, with a density $\rho_*$ given by
\begin{equation}
    \label{eq:def:mu}
    \rho_*(x) = Z_*^{-1} e^{-U(x)}, \qquad Z_* = \int_{\R^d} e^{-U_*(x)} dx < \infty
\end{equation}
where $U_*:{\R^d} \to \R_+$ is a potential that can be evaluated pointwise in ${\R^d}$, and $Z_*$ is a normalization constant, referred to as the partition function or the evidence, whose value is typically unknown (and we may wish to calculate). We assume that $U_*$ is twice-differentiable, $U_* \in C^2({\R^d})$, and grows sufficiently fast to infinity to guarantee that $Z_*<\infty$. Note that these assumptions on~$U_*$ imply that $\rho_*(x)>0$ for all $x\in{\R^d}$. 

The constructions to follow rely on the introduction of a simpler base measure $\mu_b$, for which we make the same type of assumptions, using $U_b$, $Z_b$, and $\rho_b$ to denotes its potential, normalization constant, and density. A natural choice for this base measure is the standard normal measure, $\mu_b = N(0,\Id)$.
Our  aims below will be to  characterize and sample the target density $\rho_*$ using $\rho_b$ as a reference density, assuming that we can easily draw samples from it .

Given an observable $f:{\R^d} \to \R$, we denote its expectation with respect to $\mu_*$ and $\mu_b$ respectively as $\mu_*(f)$ and $\mu_b(f)$. We also denote by $f\circ T$ the composition of $f$ with some map $T:{\R^d} \to{\R^d}$, i.e. $(f\circ T)(x) = f(T(x))$, and by $T\sharp \mu_b$ the pushforward of the measure $\mu_b$ by the map $T$.

When referring to a test function $f:\R^d\to\R$, we mean a continuously differentiable function with bounded support. 

\section{Moving Densities and Transporting their Samples}
\label{sec:transport}

We now proceed with the problem of measure transportation, which will be the foundation on which we specify means of connecting the density $\rho_b$ of the base measure to the density $\rho_*$ of the target. Much of our considerations will be based on the following classical result~\cite{santambrogio2015optimal,villani2009optimal} that shows how we can move these densities  and transport their samples consistently: 
\begin{proposition}
\label{prop:lem1}
Given some time-dependent velocity field $v:[0,1]\times{\R^d} \to \R^d$,
which we assume bounded and twice differentiable in $(t,x)$, let the PDF $\rho_t$ be the solution of the continuity equation
\begin{equation}
    \label{eq:continuity}
    \partial_t \rho_t = - \nabla \cdot (v \rho_t), \qquad \rho_{t=0} = \rho_b.
\end{equation}
Then for all test functions $f:{\R^d} \to \R$, we have
\begin{equation}
    \label{eq:push}
    \int_{\R^d} f(x) \rho_t(x) dx = \int_{\R^d} f(X_t(x)) \rho_b(x) dx
\end{equation}
where the flow map $X_t:\R^d\to{\R^d}$ is the solution to the ordinary differential equation (ODE)
\begin{equation}
    \label{eq:flow:map}
    \dot X_t(x) = v(t,X_t(x)), \qquad X_{t=0}(x) = x
\end{equation}
Conversely, if the PDF $\rho_t$ satisfies~\eqref{eq:push}, then it is the solution to~\eqref{eq:continuity}. 
\end{proposition}
\noindent
Equation~\eqref{eq:flow:map} is referred to as the probability flow ODE.

\begin{proof} The assumptions on $v_t$ guarantee global existence and uniqueness of the solution to~\eqref{eq:continuity} and \eqref{eq:flow:map}.
Taking the time derivative of~\eqref{eq:push} gives
\begin{equation}
    \label{eq:push2}
    \begin{aligned}
     \int_{\R^d} f(x) \partial_t \rho_t(x) dx &= \int_{\R^d} \dot X_t(x) \cdot \nabla f(X_t(x)) \rho_{t=0}(x) dx\\
     & = \int_{\R^d} v(t,X_t(x)) \cdot \nabla f(X_t(x)) \rho_{t=0}(x) dx \\
     & = \int_{\R^d} v(t,x) \cdot \nabla f(x) \rho_{t}(x) dx
    \end{aligned}
\end{equation}
where we used~\eqref{eq:flow:map} to get the second equality, and~\eqref{eq:push} to get the third. Equation \eqref{eq:push2} is nothing but the PDE in~\eqref{eq:continuity} written in weak form, which shows that if $\rho_t$ satisfies~\eqref{eq:push} then it solves~\eqref{eq:continuity}. Conversely, assume that $\rho_t$ satisfies~\eqref{eq:continuity} and consider
\begin{equation}
    \label{eq:dtrhot}
    \frac{d}{dt} \rho_t(X_t(x)) = \partial_t \rho_t(X_t(x)) + v(t,X_t(x)) \cdot \nabla \rho_t(X_t(x)) = - \nabla \cdot v(t,X_t(x)) \rho_t(X_t(x))
\end{equation}
where we used the chain rule along with the ODE~\eqref{eq:flow:map} to get the first equality and~\eqref{eq:continuity} to get the second. Integrating~\eqref{eq:dtrhot} in time using $\rho_{t=0} = \rho_b $ gives
\begin{equation}
    \label{eq:rhot:rho0}
    \rho_{t}(X_{t}(x)) \exp\left( \int_0^t \nabla \cdot v(\tau,X_\tau(x)) d\tau\right)= \rho_b(x) 
\end{equation}
which implies that 
\begin{equation}
    \label{eq:rhot:rho0:int}
    \int_{\R^d} f(X_t(x)) \rho_{t}(X_{t}(x)) \exp\left(\int_0^t \nabla \cdot v(\tau,X_\tau(x)) d\tau\right) dx = \int_{\R^d} f(X_t(x))\rho_b(x) dx.
\end{equation}
Using $y=X_t(x)$ as new integration variable for the integral at the left hand side and noting that the exponential factor is the Jacobian of this change of variable we arrive at~\eqref{eq:push}. 
\end{proof}

The interest of this result lies in the fact that we can in principle find a velocity field~$v$ that allows us to impose, in addition to the initial condition $\rho_{t=0} = \rho_b$, the final condition $\rho_{t=1}=\rho_*$ on the solution to~\eqref{eq:continuity}. The existence of such a $v$ is guaranteed by general theorems in optimal transport theory, see e.g.~\cite{brenier1987, caffarelli2000}, though here we do not necessarily require optimality in the sense of Monge-Amp\`ere. For example given any $\phi:{\R^d} \to \R$ that satisfies the Poisson equation
\begin{equation}
    \Delta \phi = \rho_b - \rho_*,
\end{equation}
it is easy to see that we can take
\begin{equation}
    v(t,x) = \frac{\nabla \phi(x)} {(1-t) \rho_b(x)+t\rho_*(x)}
\end{equation}
to guarantee that $\rho_t = (1-t) \rho_b + t \rho_*$ satisfies~\eqref{eq:continuity} as well as $\rho_{t=1}=\rho_*$. This is called Daracogna-Moser transport \cite{DACOROGNA19901} (which we will not use here).

With a $v$ such that $\rho_t$ satisfies both~\eqref{eq:continuity} and the final condition $\rho_{t=1}=\rho_*$  we can calculate the expectation of $f$ with respect to the target mesaure as the expectation of $f\circ X_{t=1}$ with respect to the base measure, since~\eqref{eq:push} at $t=1$ with $\rho_{t=1}=\rho_*$ can be compactly written as
\begin{equation}
    \label{eq:ideal:is}
    \mu_*(f) = \mu_b(f\circ X_{t=1})
\end{equation}
In other words, using the map $X_{t=1}$ allows one to push forward samples from the base measure~$\mu_b$ onto samples from the target~$\mu_*$. Since our working assumption is that sampling~$\mu_b$ is straightforward, using this procedure would in principle solve the problem of sampling~$\mu_*$.

Of course, the simplicity of \eqref{eq:ideal:is} is deceptive, since this equality relies on using a velocity field~$v$ which is unknown to us \textit{a~priori}. However, the identification of this velocity is amenable to  variational formulations that can be implemented in practice, and used in concert with MC methods such as importance sampling and Metropolis-Hastings MCMC. We discuss these variational formulations in Sec.~\ref{sec:learn:transp}, then how to exploit them in the context of importance sampling and Metropolis-Hastings Markov Chain Monte Carlo methods in Secs.~\ref{sec:importance} and \ref{sec:mcmc}, respectively. 

\section{Learning the Transport Variationally}
\label{sec:learn:transp}

From the results in Sec.~\ref{sec:transport} we know that we can in principle find a velocity $v_t$ such that the solution of~\eqref{eq:flow:map} gives $X_{t=1}\sharp\mu_b = \mu_*$. To obtain such a $v$, a natural idea is to minimize an objective function which enforces the final condition $\rho_{t=1} = \rho_*$.
In Secs.~\ref{sec:varia:RKL} and~\ref{sec:varia:KL} we discuss two choices for this objective and the variational problems they lead to.
As we will see, 
the solution to these variational problems 
requires the use of an adjoint method in which the probability flow ODE~\eqref{eq:flow:map} and an adjoint equation must be solved to compute the gradient of the objective \cite{chen2018, grathwohl2018scalable}. This is costly.  An alternative class of variational problems is based on building a connection that enforces the boundary conditions $\rho_{t=0}=\rho_B$ and $\rho_{t=1}= \rho_*$ exactly, which allows for direct estimation of the velocity field $v$ by solution of a quadratic regression problem, avoiding the adjoint method. Schemes with this feature are said to be `simulation-free' in the ML literature. We will discuss two methods in this class: score-based diffusion models (SBDM) \cite{ho2020, Song2019, song2021scorebased, dickstein2015}, which have revolutionized tasks in image generation and are presented in Sec.~\ref{sec:sbdm}, and stochastic interpolant methods (SIM), which are a generalization of SBDM with additional design flexibility and are presented in Sec.~\ref{sec:interp}. 

\subsection{Minimizing a Divergence Measure between $\rho_{t=1}$ and $\rho_*$}
Here we discuss variational problems in which the velocity field is obtained by minimization an $f$-divergence, taken either of $\rho_*$ from $\rho_{t=1}$ (in which case it takes the form of an expectation over $\rho_{t=1}$), or of $\rho_{t=1}$ from  $\rho_*$ (in which case it involves an expectation over $\rho_*$).  
We refer to the first as a reverse $f$-divergence and the second as a forward. We consider the benefits and tradeoffs of each below. In Secs.~\ref{sec:varia:RKL} and~\ref{sec:varia:KL} we study the reverse and forward Kullback-Leibler (KL) divergences, then in Sec.~\ref{sec:other:div} we consider other divergences such as $\chi^2$.

\subsubsection{Using the Reverse KL Divergence as Objective}
\label{sec:varia:RKL}
The reverse KL divergence reads as
\begin{equation}
    \label{eq:kl}
    \int_{\R^d} \log \left(\frac{\rho_{t=1}(x)}{\rho_*(x)}\right) \rho_{t=1}(x) dx = \int_{\R^d} \left(\log \left(\rho_{t=1}(x)\right) + U_*(x)\right) \rho_{t=1}(x) dx +C
\end{equation}
where we used $\rho_* = Z_*^{-1} e^{-U_*}$ and $C_*=\int_{\R^d} \log Z_* \, \rho_{t=1}dx = \log Z_*$ is a constant in $\rho_{t=1}$. Neglecting this constant and using the remainder as objective function leads us to:

\begin{proposition}
\label{prop:var:1}
Fix $U_*:{\R^d} \to \R_+$ and $\rho_b$, and consider the variational problem
\begin{equation}
    \label{eq:var:1}
    \begin{aligned}
    \min &\int_{\R^d} \left(\log \left(\rho_{t=1}(x)\right) + U_*(x)\right) \rho_{t=1}(x) dx \\
    \text{subject to:} & \quad \partial_t \rho_t = - \nabla \cdot (v \rho_t), \qquad \rho_{t=0} = \rho_b.
    \end{aligned}
\end{equation}
Then  all local minimizing pairs $(v,\rho_t)$ are such that $\rho_{t=1} = \rho_* = Z_*^{-1} e^{-U_*}$.
\end{proposition}

\noindent
Note that the objective in~\eqref{eq:var:1} contains the unormalized log-density $U_*$ of the target mesaure $\mu_*$, so it is only accessible if this structural information is available.
\begin{proof}
We can enforce the PDE constraint by using a Lagrange multiplier $g_t(x)$ and minimizing the extended objective
\begin{equation}
    \label{eq:extended:obj}
    \int_{\R^d} \left(\log \left(\rho_{t=1}\right) + U_*\right) \rho_{t=1} dx - \int_0^1 \int_{\R^d} g_t\left(\partial_t \rho_t + \nabla \cdot (v\rho_t) \right) dx dt 
\end{equation}
The critical points of this objective over the triple $(v,\rho_t,g_t)$ under the constraint that $\rho_{t=0} = \rho_b$ satisfy the Euler-Lagrange equations
\begin{equation}
    \label{eq:EL:1}
    \begin{aligned}
    & \partial_t \rho_t = - \nabla \cdot (v\rho_t), \qquad &&\rho_{t=0} = \rho_b,\\
    & \partial_t g_t = - v\cdot \nabla g_t, \qquad &&g_{t=1} = \log \left(\rho_{t=1}\right) + U_*,\\
    & 0 = \rho_t \nabla g_t &&
    \end{aligned}
\end{equation}
Since $\rho_t>0$, the last equation means that $g_t$ must be constant in space, which from the second equation implies that it must be constant in time too, $g_t(x) = C$. Making use of this fact at time $t=1$ allows us to write  $\log \left(\rho_{t=1}(x)\right) =  -U_*(x) +C$, with $C=-\log Z_*$ since $\rho_t$ is normalized at all times. This proves Proposition~\ref{prop:var:1} as long as a velocity field $v$ exists such that~\eqref{eq:continuity} can be satisfied together with $\rho_{t=1}=\rho_*$, which we already know is the case. 
\end{proof}

\subsubsection*{Lagrangian Formulation}

Proposition~\ref{prop:var:1} can be written in a form that involves an expectation over the base distribution, and is therefore suitable for numerical evaluation. This is a primary appeal of using $f$-divergences in reverse form. To this end, we can use~\eqref{eq:push} to rewrite the objective as
\begin{equation}
    \label{eq:new:obj}
    \begin{aligned}
    &\int_{\R^d} \left(\log \left(\rho_{t=1}(x)\right) + U_*(x)\right) \rho_{t=1}(x) dx \\
    &= \int_{\R^d} \left(\log \left(\rho_{t=1}(X_{t=1}(x))\right) + U_*(X_{t=1}(x))\right) \rho_b(x) dx
    \end{aligned}
\end{equation}
This expression can be simplified by using~\eqref{eq:rhot:rho0} with $t=1$ written as:
\begin{equation}
    \label{eq:rhot:rho0:t1}
    \rho_{t=1}(X_{t=1}(x)) = \rho_b(x) \exp\left( -\int_0^1 \nabla \cdot v(t,X_t(x)) dt\right)
\end{equation}
Inserting this equation in~\eqref{eq:new:obj} gives
\begin{equation}
    \label{eq:new:obj:2}
    \begin{aligned}
    &\int_{\R^d} \left(\log \left(\rho_{t=1}(x)\right) + U_*(x)\right) \rho_{t=1}(x) dx \\
    &= \int_{\R^d} \left( U_*(X_{t=1}(x)) -\int_0^1 \nabla \cdot v(t,X_t(x)) dt \right) \rho_b(x) dx + C_b
    \end{aligned}
\end{equation}
where $C_b = \int_{\R^d} \log \rho_b \, \rho_b dx$. Since this constant is fixed in the optimization, we arrive at the following Lagrangian reformulation of the variational problem in Proposition~\ref{prop:var:1}:

\begin{proposition}
\label{prop:var:2}
Fix $U_*:{\R^d} \to \R_+$ and $\rho_b$, and consider the variational problem
\begin{equation}
    \label{eq:var:2}
    \begin{aligned}
    \min & \int_{\R^d} \left( U(X_{t=1}(x)) -\int_0^1 \nabla \cdot v(t,X_t(x)) dt \right) \rho_b(x) dx  \\
    \text{subject to:}& \quad\dot X_t(x) = v_t(X_t(x)), \qquad X_{t=0}(x) = x.
    \end{aligned}
\end{equation}
Then, for all local minimizing pairs $(v,X_t)$, we have $X_{t=1}\sharp\mu_b = \mu_*$, where $\mu_*$ and $\mu_b$ are the measures whose densities are $\rho_*=  Z_*^{-1} e^{-U_*}$ and $\rho_b= Z_b^{-1} e^{-U_b}$, respectively.  
\end{proposition}

\begin{proof} Proposition~\ref{prop:var:2} follows from Proposition~\ref{prop:var:1} using the identity~\eqref{eq:new:obj:2} to re-express the objective function and using the equivalence between the Eulerian formulation involving $\rho_t$ and the Lagragian formulation involving $X_t$ that follows from Proposition~\ref{prop:lem1}.
\end{proof}

\subsubsection*{Data-Free Implementation}

Suppose that we approximate the function $v(t,x)$ for all $(t,x)\in[0,1]\times{\R^d}$ via a parametric representation $v^\theta(t,x)$ with parameters $\theta \in \Theta$---for example we could use a deep neural network (DNN), in which case $\theta$ contains all the network weights to adjust.  The constraint minimization problem~\eqref{eq:var:2} is amenable to solution via the adjoint method, by considering the extended objective, now viewed as a loss function in $\theta$
\begin{equation}
    \label{eq:extended:obj:2}
    \begin{aligned}
    L^\rev(\theta) = & \int_{\R^d} \left( U_*(X_{t=1}(x)) -\int_0^1 \nabla \cdot v^\theta(t,X_t(x)) dt \right) \rho_b(x) dx\\
    & - \int_0^1 \int_{\R^d}   G_t(x) \cdot (\dot X_t(x) - v^\theta(t,X_t(x))) \rho_b(x)  dx dt
    \end{aligned}
\end{equation}
where $G_t$ is a Lagrange multiplier.
The gradient of this objective then reads
\begin{equation}
    \label{eq:grad:para:obj}
    \partial_\theta L^\rev(\theta) = \int_0^1 \int_{\R^d} \left(\partial_\theta  \nabla \cdot v^\theta(t,X_t(x)) -  \partial_\theta v^\theta(t,X_t(x)) \cdot G_t(x)  \right)\rho_b(x) dx dt
\end{equation}
where the pair $(X_t(x),G_t(x))$ solves
\begin{equation}
    \label{eq:EL:2}
    \begin{aligned}
    & \dot X_t(x)= v^\theta(t,X_t(x)), \quad &&X_{t=0}(x) = x,\\
    & \dot G_t(x) = - [\nabla v^\theta(t,X_t(x))]^T G_t(x) + \nabla \nabla \cdot v^\theta(t,X_t(x)) , \quad &&G_{t=1}(x) = \nabla U_*(X_{t=1}(x)).
    \end{aligned}
\end{equation}
Note that we can solve the first equation in~\eqref{eq:EL:2} forward in time from $X_{t=0}(x)=x$ first, then use $X_{t=1}(x)$ to solve the second equation backward in time from  $G_{t=1}(x)=\nabla U_*(X_{t=1}(x))$. 

Equations~\eqref{eq:grad:para:obj} and~\eqref{eq:EL:2} can be used for optimization via stochastic gradient descent (SGD) using the following scheme:

\begin{algorithma}
\label{alg:1}
Start with an initial guess $\theta_0$ for the parameters, then for $k\ge0$: 

\noindent
(i) Draw $n$ data points $\{x^b_i\}_{i=1}^n$ from $\mu_b$ and estimate the gradient of the loss via
\begin{equation}
    \label{eq:grad:para:obj:n}
    \partial_\theta L^\rev_n(\theta_k) = \frac1n \sum_{i=1}^n \int_0^1 \left(\partial_\theta  \nabla \cdot v^{\theta_k}(t,X^{i}_t) -  \partial_\theta v^{\theta_k}(t,X^i_t) \cdot G^i_t  \right) dt
\end{equation}
where $\{(X^i_t,G^i_t)\}_{i=1}^n$ solve
\begin{equation}
    \label{eq:EL:3i}
    \begin{aligned}
    & \dot X^i_t= v^{\theta_k}(t,X^i_t), \quad &&X^i_{t=0} = x^b_i,\\
    & \dot G_t^i = - [\nabla v^\theta(t,X_t^i)]^T G_t^i+ \nabla \nabla \cdot v^\theta(t,X_t^i) , \quad &&G_{t=1}^i = \nabla U_*(X_{t=1}^i).
    \end{aligned}
\end{equation} 

\noindent
(ii) Use this estimate of the gradient to perform a step of SGD and update the parameters via
\begin{equation}
    \label{eq:sgd:is}
    \theta_{k+1} = \theta_k - h_k \partial_\theta L^\rev_n(\theta_k) 
\end{equation}
where $h_k>0$ is some learning rate.

\end{algorithma}

\noindent 
Note that we have infinite query access to $\mu_b$ so that the number $n$ of samples can be made as large as the cap on the computational cost allows, and it can also be changed from iteration to iteration.

In practice, the solution of~\eqref{eq:EL:2} requires some time discretization: this introduces some numerical error  which can be controlled by using appropriate integrators for ODEs (e.g. Runge-Kutta). The finite value of $n$ is also a source of error: it can be controlled by reducing the learning rate as the training progresses, using e.g. the scheduling used in Robbins–Monro's stochastic approximation algorithm.  In Sec.~\ref{sec:practical:is} we discuss how to integrate Algorithm~\ref{alg:1} into an importance sampling scheme.

\subsubsection{Using the Forward KL Divergence as Objective}
\label{sec:varia:KL}

The KL divergence of $\rho_*$ from $\rho_{t=1}$, henceforth referred to as the forward KL divergence, reads as
\begin{equation}
    \label{eq:kl:d}
    \int_{\R^d} \log \left(\frac{\rho_*(x)}{\rho_{t=1}(x)}\right) \rho_*(x) dx = -\int_{\R^d} \log \left(\rho_{t=1}(x)\right) \rho_*(x) dx +C'
\end{equation}
where $C'=\int_{\R^d} \log \rho_* \, \rho_* dx$ is a constant in $\rho_{t=1}$. Neglecting this constant and using the remainder as objective function leads us to:

\begin{proposition}
\label{prop:var:1:d}
Fix $\rho_b$ and $\rho_*$, and consider the variational problem
\begin{equation}
    \label{eq:var:1:d}
    \begin{aligned}
    \max &\int_{\R^d} \log \left(\rho_{t=1}(x)\right)  \rho_*(x) dx \\
    \text{subject to:} & \quad \partial_t \rho_t = - \nabla \cdot (v \rho_t), \qquad \rho_{t=0} = \rho_b.
    \end{aligned}
\end{equation}
Then, for  all local maximizing pairs $(v,\rho_t)$, we have $\rho_{t=1} = \rho_* = Z_*^{-1} e^{-U_*}$.
\end{proposition}

\noindent
Note that the objective in~\eqref{eq:var:1:d} involves an expectation over the target measure $\mu_*$ with density~$\rho_*$, but no structural information about it, such as an analytical form of the unnmormalized log-density. As a result, this objective can also be used in situations where we have data from $\mu_*$ but no model for it. 
The proof of Proposition~\ref{prop:var:1:d} is similar to that of Proposition~\ref{prop:var:1} but we include it for completeness: 
\begin{proof}
We can enforce the PDE constraint by using a Lagrange multiplier and considering instead the extended objective
\begin{equation}
    \label{eq:extended:obj:d}
    \int_{\R^d} \log \left(\rho_{t=1}\right)  \rho_* dx + \int_0^1 \int_{\R^d} h_t\left(\partial_t \rho_t + \nabla \cdot (v \rho_t) \right) dx dt 
\end{equation}
The critical points of this objective over the triple $(v,\rho_t,h_t)$ under the constraint that $\rho_{t=0} = \rho_b$ satisfy the Euler-Lagrange equations:
\begin{equation}
    \label{eq:EL:1:d}
    \begin{aligned}
    & \partial_t \rho_t = - \nabla \cdot (v \rho_t), \qquad &&\rho_{t=0} = \rho_b,\\
    & \partial_t g_t = - v\cdot \nabla h_t, \qquad &&h_{t=1} = \rho_*(x)/\rho_{t=1}(x),\\
    & 0 = \rho_t \nabla h_t &&
    \end{aligned}
\end{equation}
Since $\rho_t>0$, the last equation means that $h_t$ must be constant in space, which from the second equation implies that it must be constant in time too, $h_t(x) = C$: the final condition of the equation for $h_t$  can therefore be written as  $C\rho_{t=1}(x) =  \rho_*(x)$, with $C=1$ since $\rho_{t=1}$ must be normalized. 
\end{proof}

\subsubsection*{Lagrangian Formulation}

Proposition~\ref{prop:var:1:d} is amenable to  a Lagrangian reformulation, which is essentially a rephrasing of Proposition~\ref{prop:var:2} with the roles of $\rho_*$ and $\rho_b$ interchanged and the probability flow ODE solved in reverse.
\begin{proposition}
\label{prop:var:2:d}
Fix $U_b:{\R^d} \to \R_+$ and $\rho_*$, and consider the variational problem
\begin{equation}
    \label{eq:var:2:d}
    \begin{aligned}
    \min & \int_{\R^d} \left( U_b(\bar X_{t=0}(x)) -\int_0^1 \nabla \cdot  v(t,\bar X_t(x)) dt \right) \rho_*(x) dx  \\
    \text{subject to:}& \quad\Dot {\bar X}_t(x) = v(t,\bar X_t(x)), \qquad \bar X_{t=1}(x) = x.
    \end{aligned}
\end{equation}
Then, for all local minimizing pairs $(v,\bar X_t)$: (i) $\bar X_{t=0}\sharp\mu_* = \mu_b$ where $\mu_*$ and $\mu_b$ are the measures with density $\rho_*=  Z_*^{-1} e^{-U_*}$ and $\rho_b= Z_b^{-1} e^{-U_b}$, respectively; and (ii) the solution  to
\begin{equation}
    \label{eq:flow:d}
    \dot X_t(x) = v(t,X_t(x)), \qquad X_{t=0}(x) = x.
\end{equation}
is such that $X_{t=1}\sharp\mu_b = \mu_*$.
\end{proposition}

\noindent 
Note that the maps $X_t$ and $\bar X_t$ are inverse of each other in the sense that, for all $t\in[0,1]$, we have $\bar X_t(X_{t=1}(x)) = X_t(x)$ and $X_t(\bar X_{t=0}(x)) = \bar X_t(x)$.
\begin{proof}
Define $\bar \mu$ as the measure whose density is
\begin{equation}
    \label{eq:barrho}
    \bar \rho(x) = Z_b^{-1} \exp\left(-U_b(\bar X_{t=0}(x)) + \int_0^1  \nabla \cdot v(t,\bar X_t(x)) ds\right) 
\end{equation}
Integrating this equation against a test function $\phi:{\R^d}\to\R$ gives
\begin{equation}
    \label{eq:barrho:int}
    \int_{\R^d} \phi(x)\bar \rho(x)dx= Z_b^{-1}\int \phi(x)\exp\left(-U_b(\bar X_{t=0}(x) )+ \int_0^t  \nabla \cdot v(s,\bar X_s(x)) ds\right) dx
\end{equation}
Now use $y=\bar X_{t=0}(x)$ as new integration variable at the right-hand side, and note that: (i) $X_{t=1}(y) = x$ by definition and (ii) the Jacobian of this change of variable  precisely gives $dy = \exp\left( \int_0^1  \nabla \cdot v(t,\bar X_t(x)) dt\right) dx$. As a result
we deduce that 
\begin{equation}
    \label{eq:barrho:int:2}
    \begin{aligned}
     \int_{\R^d} \phi(x)\bar \rho(x) dx&= Z_b^{-1}\int_{\R^d}  \phi(X_{t=1}(y))\exp\left(-U_b(y)\right) dy\\
     &= 
    \int_{\R^d} \phi(X_{t=1}(y))\rho_b(y) dy
    \end{aligned}
\end{equation}
Since this relation hold for any test functions, we conclude that $\bar \mu = X_{t=1}\sharp \mu_b$ and that $\bar \rho=\rho_{t=1}$ with $\rho_t$ solution to~\eqref{eq:continuity}. Therefore the constraint minimization problem in~\eqref{eq:var:2:d} is identical to the constraint maximization problem in~\eqref{eq:var:1:d}.
\end{proof}

\subsubsection*{Practical Implementation with Data}

We can proceed as in Sec.~\ref{sec:practical:is} and approximate the function $v(t,x)$ for all $(t,x)\in[0,1]\times{\R^d}$ via a parametric representation $v^\theta(t,x)$ with parameters $\theta \in \Theta$. In this case the constraint minimization problem~\eqref{eq:var:2:d} is again amenable to solution via the adjoint method applied to the extended objective (now viewed as a loss function in $\theta$)
\begin{equation}
    \label{eq:extended:obj:2:d}
    \begin{aligned}
    L(\theta) = & \int_{\R^d} \left( U_b(\bar X_{t=0}(x)) -\int_0^1 \nabla \cdot v^\theta(t,\bar X_t(x)) dt \right) \rho_*(x) dx\\
    & - \int_0^1 \int_{\R^d}   \bar G_t(x) \cdot (\dot {\bar X}_t(x) - v^\theta(t,\bar X_t(x))) \rho_*(x)  dx dt
    \end{aligned}
\end{equation}
where $ \bar G_t$ is a Lagrange multiplier.
The gradient of this objective then reads
\begin{equation}
    \label{eq:grad:para:obj:d}
    \partial_\theta L(\theta) = \int_0^1 \int_{\R^d} \left(\partial_\theta  \nabla \cdot v^\theta(t,\bar X_t(x)) -  \partial_\theta v^\theta(t,\bar X_t(x)) \cdot \bar  G_t(x)  \right)\rho_*(x) dx dt
\end{equation}
where the pair $(\bar X_t(x),\bar G_t(x))$ solves 
\begin{equation}
    \label{eq:EL:2:d}
    \begin{aligned}
    & \Dot {\bar X}_t(x)= v^\theta(t,\bar X_t(x)), \quad &&\bar X_{t=1}(x) = x,\\
    & \Dot {\bar G}_t(x) = - [\nabla v^\theta(t,\bar X_t(x))]^T \bar G_t + \nabla \nabla \cdot v^\theta(t,\bar X_t(x)) , \quad &&\bar G_{t=0}(x) = -\nabla U_b(\bar X_{t=0}(x)).
    \end{aligned}
\end{equation}
Note that we can solve the first equation in~\eqref{eq:EL:2} backward in time from $\bar X_{t=1}(x)=x$ first, then use $\bar X_{t=0}(x)$ to solve the second equation forward in time from  $\bar G_{t=0}(x)=-\nabla U_b(\bar X_{t=0}(x))$.

To estimate empirically the gradient of the objective in~\eqref{eq:grad:para:obj:d}, we need samples from the target measure $\mu_*$ with density $\rho_*$. In Secs.~\ref{sec:importance} and~\ref{sec:mcmc}  we discuss how to generate these samples by integrating the minimization problem above it into importance sampling and Metropolis-Hastings MCMC methods, respectively. For time being, let us assume that we have at our disposal $\{x^*_i\}_{i=1}^n$ with $x^*_i\sim \mu_*$ and $n\in \N$. This is the usual setup in data science, and it allows us to perform gradient descent on the direct KL divergence via the following scheme:

\begin{algorithma}
\label{alg:2}
Start with an initial guess $\theta_0$ for the parameters, then for $k\ge0$: 

\noindent
(i) Use the data $\{x^*_i\}_{i=1}^n$ (or a minibatch thereof) to estimate the gradient of the loss via
\begin{equation}
    \label{eq:grad:para:obj:n:2}
    \partial_\theta L_n(\theta_k) = \frac1n \sum_{i=1}^n \int_0^1 \left(\partial_\theta  \nabla \cdot v^{\theta_k}(t,\bar X^i_t) -  \partial_\theta v^{\theta_k}(t,\bar X^i_t )\cdot \bar G^i_t  \right) dt
\end{equation}
where $(\bar X^i_t,\bar G^i_t)$ solve
\begin{equation}
    \label{eq:EL:2:d:i}
    \begin{aligned}
    & \Dot {\bar X}^i_t= v^{\theta_k}(t,\bar X^i_t), \quad &&\bar X^i_{t=1} = x^*_i,\\
    & \Dot {\bar G}^i_t = - [\nabla v^{\theta_k}(t,\bar X^i_t)]^T \bar G^i_t + \nabla \nabla \cdot v^{\theta_k}(t,\bar X^i_t) , \quad &&\bar G_{t=0}^i = -\nabla U_b(\bar X^i_{t=0}).
    \end{aligned}
\end{equation}

\noindent
(ii) Use this estimate of the gradient to perform a step of SGD and update the parameters via
\begin{equation}
    \label{eq:sgd:is:2}
    \theta_{k+1} = \theta_k - h_k \partial_\theta L_n(\theta_k) 
\end{equation}
where $h_k>0$ is some learning rate.

\end{algorithma}

\subsubsection{Using Other Divergences}
\label{sec:other:div}

Since ultimately, we are interested in calculating expectations via Monte Carlo, another natural objective  is the $\chi^2$-divergence of $\rho_{t=1}$ from $\rho_b$, $\int_{\R^d} \rho^2(x)/\rho_{t=1}(x) dx$. Since $\rho = Z^{-1} e^{-U}$, up to the irrelevant proportionality constant $Z^{-2}$ this objective can be written as
\begin{equation}
    \label{eq:chi2}
    \begin{aligned}
    \int_{\R^d} \frac{\rho^2(x)}{\rho_{t=1}(x)} dx &=  Z^{-2}\int_{\R^d} \frac{e^{-2U(X_{t=1}(x))}}{\rho^2_{t=1}(X_{t=1}(x))} \rho_b(x)dx\\
    &=  Z^{-2} Z_b^2\int_{\R^d} e^{-2U(X_{t=1}(x))+2U_b(x) + 2\int_0^1 \nabla \cdot v_(t,X_t(x)) dt } \rho_b(x)dx
    \end{aligned}
\end{equation}
where $\dot X_t(x) = v_(t,X_t(x))$ with $X_{t=0}(x) =x$ and we used \eqref{eq:rhot:rho0:t1} as well as $\rho_b = Z_b^{-1} e^{-U_b}$ to get the second equality. Neglecting the irrelevant proportionality constant $Z^{-2} Z_b^2$  we therefore arrive at a generalization of Proposition~\ref{prop:var:2} that we will state without proof:

\begin{proposition}
\label{prop:var:chi2}
Fix $U:{\R^d} \to \R_+$ and $\rho_b$, and consider the variational problem
\begin{equation}
    \label{eq:var:chi2}
    \begin{aligned}
    \min & \int_{\R^d} \exp\left(-2U(X_{t=1}(x))+2U_b(x) + 2\int_0^1 \nabla \cdot v(t,X_t(x)) dt\right) \rho_b(x) dx  \\
    \text{subject to:}& \quad\dot X_t(x) = v_t(X_t(x)), \qquad X_{t=0}(x) = x.
    \end{aligned}
\end{equation}
Then, for all local minimizing pairs $(v_t,X_t)$, we have $X_{t=1}\sharp\mu_b = \mu$ where $\mu$ and $\mu_b$ are the measures whose densities are $\rho=  Z^{-1} e^{-U}$ and $\rho_b= Z_b^{-1} e^{-U_b}$, respectively.
\end{proposition}
We can use Proposition~\ref{prop:var:chi2} to design a variant of Algorithm~\ref{alg:1}. Note that the exponential factor in~\eqref{eq:var:chi2} is exactly the square of the importance sampling weights that we will define in~\eqref{eq:wk:0} and controls the efficiency of the importance sampling method. It is therefore natural to the expectation of the square of these weights as loss function. Note we can also use the $\chi^2$-divergence of $\rho_b$ from $\rho_{t=1}$, $\int_{\R^d} \rho_{t=1}^2(x)/\rho(x) dx$ which leads 
\begin{equation}
    \label{eq:chi2:R}
    \int_{\R^d} \frac{\rho^2_{t=1}(x)}{\rho(x)} dx =   Z^2Z_b^{-2}\int_{\R^d} e^{2U(x)-2U_b(\bar X_{t=0}(x)) - 2\int_0^1 \nabla \cdot v_t(\bar X_t(x)) dt } \rho(x)dx
\end{equation}
where $\Dot {\bar X}_t(x) = v_t(\bar X_t(x))$ with $X_{t=1}(x) =x$ and the constant $Z^{2} Z_b^{-2}$ can again be neglected.

\subsection{Simulation-Free Objectives}
We now consider an alternative viewpoint on formulating a variational problem for learning a transport map between $\rho_b$ and $\rho_*$. The divergence-based objectives presented in the previous methods, while well-motivated from the perspectives of maximum likelihood estimators, require simulating the probability flow on candidate samples. This can be prohibitively costly for high-dimensional problems. We describe next a class of methods which, rather than having a path in the space of measures be chosen implicitly by minimizing e.g. \eqref{eq:var:1}, explicitly fix $\rho_t$ and \textit{regress} the resulting drift coefficients governing the probability flow ODE. This perspective underlies the success of score-based diffusion models which we discuss next.

\subsubsection{Score-based diffusion models}
\label{sec:sbdm}

Score-based diffusion models \cite{Song2019 ,ho2020, song2021scorebased} are generative models that transform data into noise in a diffusive process, and then learn how to invert the diffusion to generate new data. These methods have gained a lot of popularity for their success in the context of image generation. One of their main advantages is that the learning involved is a standard quadratic regression problem that does not require any adjoint method for its solution.  

\subsubsection*{Transport by Fokker-Planck Equations and Forward SDEs}
Score-based difusion models are based on variants of the Ornstein-Uhlenbeck process, whose probability density function satisfies the Fokker-Planck equation
\begin{equation}
    \partial_t \rho_t = \nabla\cdot(x\rho_t) +\Delta\rho_t. 
    \label{fp}
\end{equation}
The solution to this equation for the initial condition $\rho_{t=0} = \rho_*$ has the property that is converges to the density of the standard normal distribution as $t\to\infty$. This can be seen from the explicit solution to~\eqref{fp} that is given by
\begin{equation}
    \rho_t(x) = \int _{\R^d }\rho^c_t(x|y) \rho_*(y) dy
    \label{fp2}
\end{equation}
where we defined the conditional probability density 
\begin{equation}
    \rho^c_t(x|y) = \left(2\pi (1-e^{-2t})\right)^{-d/2} \exp\left( -\frac{|x-y e^{-t}|^2}{2 (1-e^{-2t}))}\right)
    \label{fp2:c}
\end{equation}
Since $\rho^c_t(x|y) \to (2\pi)^{-d/2}e^{-\frac12|x|^2}$ for all $y\in \R^d$ as $t\to\infty$,  \eqref{fp2} implies that $\rho_t(x) \to (2\pi)^{-d/2}e^{-\frac12|x|^2}$ in this limit too. Alternatively, we can also derive this result from the stochastic differential equation (SDE) associated with the FPE~\eqref{fp}:
\begin{equation}
    dX^\fwd_t = -X^\fwd_tdt +\sqrt{2}\,dW_t,
    \label{sde1}
\end{equation}
where $W_t$ is a Wiener process -- $\rho_t$ is the probability density of $X^\fwd_t$. The solution to this SDE for the initial condition $X^\fwd_{t=0}=x$ is:
\begin{equation}
    X^\fwd_t = x\,e^{-t}+\sqrt{2}\int_0^t  e^{-t+s}dW_s.
    \label{sde2}
\end{equation}
A simple calculation using (this is It\^o isometry)
\begin{equation}
    \E\left[\left(\sqrt{2}\int_0^te^{-t+s}dW_s\right)^2\right]= 2\int_0^te^{-2t+2s}ds \Id= (1-e^{-2t})\Id,
\end{equation}
indicates that at any time $t$, the law of~\eqref{sde2} is Gaussian with mean $x e^{-t}$ and covariance $(1-e^{-2t})\Id$, and it converges to that of the standard normal variables as $t\to\infty$ for all initial condition~$x$. The key insight of score-based diffusion models is to use either \eqref{fp} or \eqref{sde1} in reverse (in a sense to be made precise next) to turn Gaussian samples into new samples from~$\rho$ instead of the other way around. Let us see next  how this can be done. 

\subsubsection*{Continuity Equation and Score}

The FPE \eqref{fp} can be recast into the continuity equation~\eqref{eq:continuity} for the velocity field
\begin{equation}
    v(t,x)=x-\nabla \log\rho_t(x).
    \label{sde4}
\end{equation}
The quantity $\nabla \log\rho_t(x)$ is referred to as the `score'. This quantity is unknown to us, but it can in principle be learned at any time $t$ since  is the unique minimizer over $s$ of the 
\begin{equation}
    E[s(t)] = \int_{\R^d}\left|s(t,x) -\nabla\log\rho_t(x) \right|^2 \rho_t(x)dx.
    \label{score1}
\end{equation}
Expanding the square, using the identity $\nabla \log \rho_t \, \rho_t = \nabla \rho_t$ in the cross term, and  integrating this term  by parts, we can reexpress this objective as~\cite{hyvarinen05a}
\begin{equation}
\begin{aligned}
    E[s(t)] &= \int_{\R^d}\left[ |s(t,x)|^2 -2 s(t,x) \cdot \nabla\log\rho_t(x) \right] \rho_t(x)dx + C_t\\
    &= \int_{\R^d}\left[ |s(t,x)|^2 +2 \nabla \cdot s(t,x) \right] \rho_t(x)dx + C_t
\end{aligned}
    \label{score2}
\end{equation}
where $C_t=\int_{\R^d}\left|\nabla\log\rho_t(x) \right|^2 \rho_t(x)dx$ is a constant independent of $s$. Neglecting this irrelevant constant, we can write the first integral at the right-hand side of \eqref{score2} as 
\begin{equation}
\begin{aligned}
    \E\left[ |s(t,X^\fwd_t)|^2 +2 \nabla \cdot s(t,X^\fwd_t) \right]
\end{aligned}
    \label{score3}
\end{equation}
where $X^\fwd_t$ is given in \eqref{sde2} and the expectation is taken independently over $x\sim \mu_*$ and  the Wiener process. 

\subsubsection*{Simulation-Free Implementation with Data}

Assuming that we have access to a data set $\{x^*_i\}_{i=1}^n$ drawn from the target measure $\mu_*$, the expectation~\eqref{score3} can be estimated empirically upon noticing that the law of $X^\fwd_t$ conditional on  $X^\fwd_{t=0}= x^*_i$ is $N(x^*_ie^{-t}, (1-e^{-2t}) \Id)$. 
If in addition we approximate the score by a function $s^\theta_t$ in some rich parametric class, \eqref{score3} becomes the following empirical loss for $\theta$
\begin{equation}
    \label{eq:loss:sbdm}
    L_{\text{\sc SBDM}}(\theta) = \frac1n \sum_{i=1}^n \left[ \big|s^\theta\big(t,x^*_ie^{-t}+\sqrt{ 1-e^{-2t}}\xi_i\big)\big|^2 +2 \nabla \cdot s^\theta\big(t,x^*_ie^{-t}+\sqrt{ 1-e^{-2t}}\xi_i\big) \right]
\end{equation}
where the variables $\xi_i$ are independent $N(0,\Id)$. This loss is valid for any $t\in[0,\infty)$, but we can also takes its expectation over random times~$t$ (e.g. exponentially distributed) to learn $s^\theta(t,x)$ globally as a function of $(t,x)$. 

Performing SGD on~\eqref{eq:loss:sbdm} to train the parameters $\theta$ is straightforward, and this  leads to an approximation of the score $s^\theta(t,x)$ that can be used in~\eqref{sde4} to get approximation of the velocity given by $v^\theta(t,x) = x-s^\theta(t,x)$. This velocity field can then be used in the probability flow ODE~\eqref{eq:flow:map}  to push back samples from $\mu_b$ generated at some $t=T>0$  onto new samples from $\mu_*$ at time $t=0$ by integrating this ODE backward in time. Note that this introduces a small bias since we should take the limit as $T\to\infty$ to guarantee consistency: in practice this bias can be efficiently controlled by increasing~$T$ since convergence of $X^\fwd_t$ towards Gaussianity is exponential in time.  

\subsubsection*{Reverse-Time SDE}

In SBDM, samples from $\mu_*$ are usually not generated using the probability flow ODE~\eqref{eq:flow:map}, but rather by using the reverse-time SDE
\begin{equation}
    dX^\rev_t = X^\rev_t dt +2 \nabla \log \rho_{T-t}(X^\rev_t) dt +\sqrt{2}\,dW_t.
    \label{sde1:rev}
\end{equation}
This SDE can be obtained by noticing that, if $\rho_t$ solves the FPE~\eqref{fp}, then $\rho^R_t = \rho_{T-t}$ solves
\begin{equation}
\begin{aligned}
    \partial_t \rho^\rev_t = -\partial_t  \rho_{T-t} &= -\nabla\cdot(x\rho_{T-t}) -\Delta\rho_{T-t} \\
    & =  -\nabla\cdot([x+2\nabla \log \rho_{T-t}] \rho_{T-t} ) +\Delta\rho_{T-t}\\
    & =  -\nabla\cdot([x+2\nabla \log \rho_{T-t}] \rho^\rev_t ) +\Delta\rho^\rev_{t}
\end{aligned}
    \label{fp:rev}
\end{equation}
The solution $\rho_t^\rev$ of this FPE is the density of the process $X_t^\rev$ satisfying the reverse-time SDE~\eqref{sde1:rev}.
The solution to this SDE therefore has the property that, if $X^\rev_{t=0} \sim \mu_{t=T}$, where $\mu_t$ is the measure whose density $\rho_t$ satisfies the FPE~\eqref{fp}, then $X^\rev_{t=T} \sim \mu_*$. By choosing $T$ large enough and replacing $\nabla \log \rho_{t}(x)$ by  $s^\theta(t,x)$ with $\theta$ obtained by minimizing~\eqref{eq:loss:sbdm}, we can therefore use Gaussian samples for $ X^\rev_{t=0}$ and turn them into samples $X^\rev_{T}$ whose measure is approximately the target~$\mu$.

\begin{remark}
    In general the density $\rho_t$ given in~\eqref{fp2} is not available explicitly. One exception is when the target $\mu_*$ is a  mixture of Gaussian measures, and $\rho_*$ is  given by
\begin{equation}
    \rho_*(x) = \sum_{j=1}^{N}p_j (2\pi)^{-d/2} (\det C_j)^{-1/2} \exp\left(-\tfrac12 (x-b_j)^T C_j^{-1}  (x-b_j)\right),
\end{equation}
where $N$ is the number of modes, $p_j$ is the probability of mode $j$,  with $\sum_{j=1}^N p_j =1$ and $p_j>0$, and $b_j$ and $C_j$ are respectively the mean and covariance of mode $j$. In this case
\begin{equation}
    \rho_t(x)  = \sum_{j=1}^{M}p_j (2\pi)^{-d/2} (\det \hat C_j(t))^{-1/2} \exp\left(-\tfrac12 (x-\hat b_j(t))^T \hat C_j^{-1}(t)  (x-\hat b_j(t))\right),
    \label{gm1}
\end{equation}
where
\begin{equation}
    \label{eq:m:C}
    \hat m_j(t) = m_j e^{-t}, \qquad \hat C_j(t) = C_je^{-2t}+(1-e^{-2t})\Id.
\end{equation}
The score $s_t=\nabla \log \rho_t$ is therefore given by
\begin{equation}
\begin{aligned}
    s(t,x) = -\rho_t^{-1}(x)&\sum_{j=1}^{M}p_j (2\pi)^{-d/2} (\det \hat C_j(t))^{-1/2} \\
&\times\hat C_j^{-1}(t)  (x-\hat b_j(t)) \exp\left(-\tfrac12 (x-\hat b_j(t))^T \hat C_j^{-1}(t)  (x-\hat b_j(t))\right).
\end{aligned}
\end{equation}
Note that this formula remains valid for all $t>0$ even if we let $C_j\to0$, i.e. if $\rho_t$ degenerates into a sum of $M$ point masses at $x=b_j$. Setting $b_j=x_j$ and $M=n$ gives the exact solution for $\rho_t$ and $s_t$ on the data set $\{x_i\}_{i=1}^n$. Of course, this solution  simply memorizes the data, so it leads to overfitting. To avoid this issue, the minimization of~\eqref{eq:loss:sbdm} must be performed subject to some regularization.
\end{remark} 

\subsubsection{Stochastic Interpolants}
\label{sec:interp}

The key insight behind score-based diffusion models is to first build a time-continuous connection between the base and the target measures that is easy to sample (given data from both the base and the target) and leads to dynmical map whose velocity field can be estimated via quadratic regression on these samples. Stochastic interpolants~\cite{albergo2023building} provide us with an alternative way to define this connection in measure space. Specifically, we define the process\footnote{%
In~\cite{albergo2023building}, more general stochastic interpolants are defined that are nonlinear in $(x^b,x^*)$ and include a latent variable.}
\begin{equation}
    \label{eq:stoch:interp}
    x_t = \alpha(t) x^b + \beta(t) x^*, \qquad t\in [0,1] 
\end{equation}
where $\alpha(t)$ and $\beta(t)$ are differentiable functions that satisfy $\alpha(0)=\beta(1) = 1$, $\alpha(1)=\beta(0) = 0$, and $\alpha^2(t) + \beta^2(t)>0$ for all $t\in [0,1]$; and $(x^b,x^*)$ are jointly drawn from a probability measure $\pi$  that marginalizes on the base and the target measures: 
\begin{equation}
    \label{eq:margin}
    \pi(dx,\R^d) = \mu_b(dx), \qquad \pi(\R^d,dx) = \mu(dx)
\end{equation}
For example, we could take $\alpha(t) = 1-t$, $\beta(t) = t$, and $\pi(dx,dy) = \mu_b(dx) \mu_*(dy)$, though other choices are possible.

By construction we have $x_{t=0} = x^b \sim \mu_b$ and $x_{t=1} = x^* \sim \mu_*$, and since the process $x_t$ is continuous and differentiable in time, its measure inherits these properties and interpolates between the base $\mu_b$ and the target $\mu_*$. If we make some minimal assumptions on the measure $\pi$, for example that it has a density on $\R^d\times \R^d$ that is sufficiently smooth, then the measure of the process $x_t$  also has a density, say $\rho_t$, and this density satisfies the continuity equation~\eqref{eq:continuity}  for some specific velocity~$v_t$. To see that,  notice that, by definition 
\begin{equation}
    \label{sec:interp:pdf:0}
        \int_{\R^d} f(x) \rho_t(x) dx = \int_{\R^d\times\R^d} f(\alpha(t) x^b + \beta(t) x^*) \pi(dx^b,dx^*)
\end{equation}
for any test function $f:\R^d\to \R$.
Taking the derivative of this equality using the chain rule and $\dot x_t = \dot\alpha(t) x^b + \dot \beta(t) x^*$ gives 
\begin{equation}
    \label{sec:interp:pdf:1}
    \begin{aligned}
    \int_{\R^d} f(x) \partial_t \rho_t(x) dx & =  \int_{\R^d\times \R^d} \left(\dot\alpha(t) x^b + \dot \beta(t) x^*\right) \cdot \nabla f(\alpha(t) x^b + \beta(t) x^*) \pi(dx^b,dx^*)\\
        & = \int_{\R^d} \E\big[\left(\dot\alpha(t) x^b + \dot \beta(t) x^*\right) \cdot \nabla f(\alpha(t) x^b + \beta(t) x^*) \big|x_t = x\big]  \rho_t(x) dx\\
        & = \int_{\R^d} \E\left[\dot\alpha(t) x^b + \dot \beta(t) x^* |x_t = x\right]  \cdot \nabla f(x) \rho_t(x) dx
    \end{aligned}
\end{equation}
where $\E[\cdot | x_t=x]$ denotes expectation over $\pi$ conditional on $x_t = x$ with $x_t$ is given by~\eqref{eq:stoch:interp}.
This is the weak form of the continuity equation~\eqref{eq:continuity} for the velocity
\begin{equation}
    \label{eq:vt:interp}
    v(t,x) = \E\left[\dot\alpha(t) x^b + \dot \beta(t) x^* |x_t = x\right].
\end{equation}
It is easy to see that this velocity field also minimizes over all $v$ the objective
\begin{equation}
    \label{eq:vt:interp}
    \int_0^1 \E\big[ |v(t,x_t)|^2 -2 (\dot\alpha(t) x^b + \dot \beta(t) x^*)\cdot v(t,x_t)\big] dt,
\end{equation}
where $x_t$ is given by~\eqref{eq:stoch:interp} and $\E$ denotes expectation over $\pi$. Therefore, if we are given a data set $\{x_i^b,x_i^*\}_{i=1}^n$ of independent samples from $\pi$, and we wish to approximate $v(t,x)$ via some parametric $v^\theta(t,x)$ with $\theta\in \Theta$, these parameters can be estimated by minimizing the empirical loss
\begin{equation}
    \label{eq:loss:interp}
    L_{\text{interp}}(\theta) = \frac1n \sum_{i=1}^n \int_0^1\left[ |v^\theta(t,x^i_t)|^2 -2 (\dot\alpha(t) x_i^b + \dot \beta(t) x_i^*)\cdot v^\theta_t(x^i_t)\right]dt
\end{equation}
where $x^i_t = \alpha(t) x_i^b + \beta(t) x_i^*$. Using this loss therefore reduces the estimation of the velocity to a quadratic regression in problem that is simulation-free like the one used in SBDM. Similar ideas were presented in \cite{lipman2022, liu2022}. Of course, to use~\eqref{eq:loss:interp} we need data. How to obtain these by Monte-Carlo sampling will be discussed next in Secs.~\ref{sec:importance} and~\ref{sec:mcmc}.

\section{Importance Sampling with and without Transport}
\label{sec:importance}
In this section we introduce importance weighting in relation to computing expectations under our target measure, $\mu_*(f)$. We will use this to then describe a reweighting scheme with respect to the transported measures described in the previous section.

\subsection{Basic Identity without Transport} 
\label{sec:basic:is}
Vanilla importance sampling is based on the following identity:
\begin{equation}
    \label{eq:is}
    \mu(f) = \frac{\mu_b(f w)}{\mu_b(w)} \qquad \text{where}  \qquad w(x) = e^{-U(x) + U_b(x)}.
\end{equation}
Therefore, if $\{x_i\}_{i\in \N}$ are independent samples from $\mu_b$, the law of large numbers guarantees that
\begin{equation}
    \label{eq:lln:is}
    \mu(f) = \lim_{n\to\infty} \frac{\sum_{i=1}^n f(x_i) w(x_i)}{\sum_{i=1}^n w(x_i)}
\end{equation}
The quantity under the limit at the right hand side is a biased estimator of $\mu(f)$ since it relies on a ratio of empirical averages, and it is easy to see that the one appearing at the denominator is an estimate of the ratio of partition functions
\begin{equation}
    \label{eq:ratio:Z}
    \frac{Z}{Z_b} = \mu_b(w) =  \lim_{n\to\infty} \frac1n \sum_{i=1}^n w(x_i).
\end{equation}
To estimate the quality of an estimator based on using \eqref{eq:is} with a finite $n$ approximation of the limit, let us look at the variance $\mu_b(w^2)-|\mu_b(w)|^2$ of the estimator for $Z/Z_b$ obtained from \eqref{eq:ratio:Z}. A simple calculation shows that
\begin{equation}
    \label{eq:mub2}
    \frac{\mu_b(w^2)}{|\mu_b(w)|^2} = \frac{Z_b}{Z^2}\int_{\R^d} e^{-2U(x)+U_b(x)}dx.
\end{equation}
This factor is always larger than $1$ by Jensen's inequality, and typically much larger, even infinite. For example if $\mu_b= N(0,\Id)$ and $\mu= N(0,\alpha^{-1}\Id)$ with $\alpha\in(0,\infty)$, \eqref{eq:mub2} is infinite if $\alpha\le\frac12$ and $ \mu_b(w^2)/|\mu_b(w)|^2 = O(\alpha^{d/2})$ as $\alpha\to\infty$. That is, when $d$ is large, an estimator based on using \eqref{eq:ratio:Z} a finite $n$ is only usable in practice if $\alpha$ is close to 1 in this simple example. For this reason, the estimator for $\mu(f)$ based on~\eqref{eq:lln:is} is not used in practice, and rather replaced by variants where a connection is constructed between $\mu_b$ and $\mu$: this idea is at the core of methods based on replica exchange and  thermodynamic integration or nonequilibrium sampling schemes such as Neal's  annealed sampling importance (AIS) method. Next we show how to use transport to improve importance sampling estimators.

\subsection{Importance Sampling with Transport}
\label{sec:practical:is}

To do importance sampling with transport, instead of using bare samples from the base $\mu_b$ in the estimator based on~\eqref{eq:lln:is}, we first transport them to a measure that is hopefully closer to the target~$\mu_*$. This operation relies on the following result:

\begin{proposition}
\label{prop:expect:is:0}
Let $v$ be a velocity field that is bounded and twice differentiable in $(t,x)$, and let $X_t$ be the solution to the associate probability flow ODE~\eqref{eq:flow:map}:
\begin{equation}
    \label{eq:prob:flow}
    \dot X_t(x) = v(t,X_t(x)), \qquad X_{t=0}(x) =x.
\end{equation} 
Then, given any test function $f:{\R^d} \to\R$, we have
\begin{equation}
    \label{eq:expect:is:0}
    \mu(f) = \frac{ \mu_b([f\circ X_{t=1}] w_v)}{\mu_b(w_v)} 
\end{equation}
where
\begin{equation}
    \label{eq:wk:0}
    w_v(x) = \exp\left( -U(X_{t=1}(x)) + U_b(x) + \int_0^1 \nabla \cdot v(t,X_t(x))dt\right)
\end{equation}
\end{proposition}

\noindent 
Note also that we can build an estimator based on~\eqref{eq:expect:is:0}: If  $\{x^b_i\}_{i=1}^n$ are $n\in \N$ independent samples of $\mu_b$, we have 
\begin{equation}
    \label{eq:expect:is:2}
    \mu(f) = \lim_{n\to\infty} \frac{\sum_{i=1}^n f(X_{t=1}(x^b_i)) w_v(x^b_i)}{\sum_{i=1}^n w_v(x^b_i)} 
\end{equation}

\begin{proof}
If $\rho_t$ is the solution to the continuity equation in~\eqref{eq:var:1}, we have
\begin{equation}
    \begin{aligned}
    \int_{\R^d} f(x) e^{-U(x)}  dx &= \int_{\R^d} [f(x) e^{-U(x)} /\rho_{t=1}(x)] \rho_{t=1}(x) dx\\
    &= \int_{\R^d} [f(X_{t=1}(x)) e^{-U(X_{t=1}(x))} /\rho_{t=1}(X_{t=1}(x))] \rho_b(x) dx\\
    & = Z_b \int_{\R^d} f(X_{t=1}(x)) w_v(x) \rho_b(x) dx.
    \end{aligned}
\end{equation}
where we used~\eqref{eq:rhot:rho0:t1} and the definition of $w_v$ in~\eqref{eq:wk:0} to get the last equality. This implies that
\begin{equation}
    \begin{aligned}
    Z \mu(f) = Z_b \mu_b([f\circ X_{t=1}] w_v) \quad \text{and} \quad Z = Z_b \mu_b(w_v)
    \end{aligned}
\end{equation}
This proves~\eqref{eq:expect:is:0}. 
\end{proof}

This proof shows that, if $v$ is such that the solution to~\eqref{eq:prob:flow} satisfies $X_{t=1}\sharp \mu_b=\mu$, then $w_v(x) = Z/Z_b$ for all $x\in \R^d$ -- this is the ideal case. However Proposition~\ref{prop:expect:is:0} also holds for  maps $X_t$ that are not perfect, for example those obtained practically by minimizing the objective~\eqref{eq:grad:para:obj} based on the reversed KL divergence using Algorithm~\ref{alg:1}. 

\subsection{Learning with Importance Sampling }
\label{sec:learning:is}

Interestingly, Proposition~\ref{prop:expect:is:0} can also be used to estimate the losses based on the direct KL divergence or those arising in SBDM or with stochastic interpolants. To see how, let us consider first Algorithm~\ref{alg:2} for the optimization of the velocity using the direct KL divergence as loss. In step (i) of this algorithm, instead of using the estimate \eqref{eq:grad:para:obj:n:2} that relies on data that we may not have, we can use a batch $\{x_i^b\}_{i=1}^n$ drawn from $\mu_b$ and importance sampling with transport to reweight these data using the map as we learn it. This leads to: 

\begin{algorithma}
\label{alg:3}
Start with an initial guess $\theta_0$ for the parameters, then for $k\ge0$: 

\noindent
(i) Draw  data $\{x^b_i\}_{i=1}^n$ from $\mu_b$ and estimate the gradient of the loss via
\begin{equation}
    \label{eq:grad:para:obj:is:a}
    \partial_\theta L_n(\theta_k) = \sum_{i=1}^n \tilde w_i \int_0^1  \left(\partial_\theta  \nabla \cdot v^{\theta_k}(t,X^i_t) -  \partial_\theta v^{\theta_k}(t,X^i_t )\cdot G^i_t  \right)  dt
\end{equation}
where $w_i$ is given by
\begin{equation}
    \label{eq:wk:0}
    \tilde w_i = \frac{w_i}{\sum_{j=1}^nw_j} \quad \text{with} \quad w_i = \exp\left( -U(X_{t=1}^i) + U_b(x_b^i) + \int_0^1 \nabla \cdot v^{\theta_{k-1}}(t,X_t^i)dt\right).
\end{equation}
and $(X^i_t,  G^i_t)$ solve
\begin{equation}
    \label{eq:EL:2:d:i:a}
    \begin{aligned}
    &\dot X^i_t = v^{\theta_{k}}(t,X^i_t) && X^i_{t=0} = x_b^i\\
    & \dot G^i_t = - [\nabla v^{\theta_k}(t,X^i_t)]^T G^i_t + \nabla \nabla \cdot v^{\theta_k}(t,X^i_t) , \quad &&\bar G_{t=0}^i = -\nabla U_b(x_b^i).
    \end{aligned}
\end{equation}

\noindent
(ii) Use this estimate of the gradient to perform a step of SGD and update the parameters via
\begin{equation}
    \label{eq:sgd:is:2}
    \theta_{k+1} = \theta_k - h_k \partial_\theta L_n(\theta_k) 
\end{equation}
where $h_k>0$ is some learning rate.

\end{algorithma}
\noindent
To justify this algorithm, notice that, by \eqref{eq:expect:is:0} in Proposition~\ref{prop:expect:is:0}, we get a consistent estimator of the expectation in~\eqref{eq:grad:para:obj:d}  in the limit as $n\to\infty$ by using $(\bar X_t^i,\bar G_t^i)$ instead of $(X_t^i, G_t^i)$ in \eqref{eq:grad:para:obj:is:a}, where the triple $(X^i_t, \bar X^i_t,\bar G^i_t)$ solves
\begin{equation}
    \label{eq:EL:2:d:i:a:0}
    \begin{aligned}
    &\dot X^i_t = v^{\theta_{k}}(t,X^i_t) && X^i_{t=0} = x_b^i\\
    & \Dot {\bar X}^i_t= v^{\theta_k}(t,\bar X^i_t), \quad &&\bar X^i_{t=1} = X^i_{t=1},\\
    & \Dot {\bar G}^i_t = - [\nabla v^{\theta_k}(t,\bar X^i_t)]^T \bar G^i_t + \nabla \nabla \cdot v^{\theta_k}(t,\bar X^i_t) , \quad &&\bar G_{t=0}^i = -\nabla U_b(\bar X^i_{t=0}).
    \end{aligned}
\end{equation}
Since $\bar X_t^i = X_t^i$, these equations reduce to \eqref{eq:EL:2:d:i:a} and the estimator to~\eqref{eq:grad:para:obj:is:a}.

We can proceed similarly for the loss~\eqref{eq:loss:sbdm} used in SBDM and \eqref{eq:loss:interp} used with stochastic interpolant. Assuming that we wish to evaluate them at the current value~$\theta_k$  of the parameters, they   respectively become
\begin{equation}
    \label{eq:loss:sbdm:is}
    L_{\text{\sc SBDM}}(\theta_k) = \sum_{i=1}^n \tilde w_i \left[ \big|s^{\theta_k}\big(t,y^i_t\big)\big|^2 +2 \nabla \cdot s^{\theta_k}\big(t,y^i_t\big) \right] 
\end{equation}
where $y^i_t = X^i_{t=1}e^{-t}+\sqrt{ 1-e^{-2t}}\xi_i$,
and
\begin{equation}
    \label{eq:loss:interp:is}
    \begin{aligned}
    L_{\text{interp}}(\theta_k) &= \sum_{i=1}^n  \tilde w_i \int_0^1 \big[ |v^{\theta_k}(t,z^i_t)|^2 -2  (\dot\alpha(t) x_i^b + \dot \beta(t) X^i_{t=1})\cdot v^{\theta_k}(t,z_t^i) \big] dt 
    \end{aligned}
\end{equation}
where $z^i_t = \alpha(t) x_i^b + \beta(t) X^i_{t=1}$.
In both~\eqref{eq:loss:sbdm:is} and \eqref{eq:loss:interp:is} $\tilde w_i$ is given by \eqref{eq:wk:0} and $X^i_t$ solves the first equation in \eqref{eq:EL:2:d:i:a} (with $v^{\theta_{k}}(t,x) = x - s^{\theta_{k}}(t,x)$ for SBDM). Note that these empirical losses are no longer simulation-free, though their evaluation requires only one ODE solve per sample.

\section{Assisting Markov Chain Monte Carlo}
\label{sec:mcmc}

\subsection{Metropolis-Hastings MCMC}
\label{sec:mhmcmc} The basic idea behind MCMC methods is to create a Markov sequence $\{x_i\}_{i\in \N_0}$ ergodic with respect to the taget distribution $\mu$, so that the time average $S_n$ of an observable~$f$~along this sequence converges towards its expectation with respect to this target:
\begin{equation}
    \label{eq:time:aveg}
    \lim_{n\to\infty} S_n(f) = \mu(f) \qquad \text{where} \qquad S_n(f) = \frac1n \sum_{i=1}^n f(x_i)
\end{equation}
In the context of Metropolis-Hastings MCMC, this is achieved by (i) proposing a new sample $\tilde x$ using some transition probability kernel $Q^{x_i}(dy) = q^{x_i}(y) dy$; (ii) accepting $\tilde x$ as new state  $x_{i+1}=\tilde x$ with probability
\begin{equation}
    \label{eq:accept:reject}
    a(x_i,\tilde x) = \min\left( \frac{\rho_*(\tilde x) q^{\tilde x}(x_i) }{\rho_*(x_i) q^{x_i}(\tilde x)},1\right)
\end{equation}
and setting $x_{i+1} = x_i$ if the proposal is rejected. It is easy to see that the corresponding transition kernel of the chain then reads
\begin{equation}
    \label{eq:transition}
    P^x(dy) = a(x,y) q^x(y) dy - b(x) \delta_x(dy) \quad \text{with} \quad b(x) = \int_{\R^d} a(x,y) q^x(y) dy
\end{equation}
which guarantees that the chain is in detailed-balance with respect to the target distribution~$\mu$:
\begin{equation}
    \label{eq:detailed}
    d\mu(x) P^x(dy) = d\mu(y) P^y(dx)
\end{equation}
This property, together with a condition of irreducibility (e.g. if $q^x(y)>0$ for all $x,y\in {\R^d}$) guarantees ergodicity of the chain, i.e. \eqref{eq:time:aveg} holds.

Unlike the importance sampling identity~\eqref{eq:lln:is}, \eqref{eq:time:aveg} is an unbiased estimator that contains no weights. The price to pay, however, is that the samples $\{x_i\}_{i\in \N_0}$ in the sequence are no longer independent, and time-correlation effects along the chain affect the error made by using $S_n(f)$ with a finite $n$ to estimate~$\mu(f)$. A direct calculation shows that  
\begin{equation}
    \label{eq:clt:mcmc}
    \mu\left( |S_n(f)-\mu(f)|^2\right) \sim \frac1n \mu\left([f-\mu_*(f)] u\right) \qquad \text{as} \ \ n\to\infty
\end{equation}
where $u(x)$ is the solution to  
\begin{equation}
    \label{eq:poisson}
     u(x) - \int_{\R^d} u(y) P^x(dy) = f(x)-\mu_*(f)
\end{equation}
that minimizes $\mu_*(u^2)$.
Because the spectral radius of $P^x(dy)$ is 1, we always have 
\begin{equation}
    \label{eq:bound:var}
    \mu\left([f-\mu_*(f)] u\right) \ge \mu_*\left(|f-\mu_*(f)|^2 \right) = \text{var}_* (f)
\end{equation}
and it is easy to see that the only way to achieve this bound is to take $P^x(dy) = \mu_*(dy)$. This ideal choice amounts to generating independent sample from $\mu$, i.e. it corresponds to using $q^x(y) = \rho_*(y)$, and it is easy to see  from~\eqref{eq:accept:reject} that $a(x_i,\tilde) =1$ in that case, i.e. all proposals are accepted. However, the bound in~\eqref{eq:bound:var} is far from being achieved in general because it is hard to design chains whose correlation time is short:  if we pick $q^x(y)$ such that $\tilde x$ is close to $x_i$, this proposal is likely to be accepted, but the chain will move slowly; and if we naively pick $q^x(y)$ such that $\tilde x$ is far from $x_i$, this proposal will be likely to be rejected, which also results in no motion, $x_{i+1}=x_i$.

Next we discuss how to use any map to perform independent MH-MCMC (Sec.~\ref{sec:map:mcmc}), then how learn a  a map such that $X_{t=1}\sharp \mu_b$ is approximately the target $\mu_*$ so that we get closer to the ideal chain with acceptance probability $a(x_i,x_*) =1$ (Sec.~\ref{sec:leran:mcmc}).

\subsection{Map-Assisted MH-MCMC}
\label{sec:map:mcmc}

Our next result shows how to perform independent MH-MCMC sampling by using as proposals samples from the base that have been pushed trough a map: 

\begin{proposition}
\label{prop:expect:mcmc:0}
Let $v$ be a velocity field that is bounded and twice differentiable in $(t,x)$, and let $(X_t,\bar X_t)$ be the solution to the forward and backward probability flow ODEs:
\begin{equation}
    \label{eq:prob:flow:f:b}
    \begin{aligned}
        \dot X_t(x) &= v(t,X_t(x)), \qquad &&X_{t=0}(x) =x,\\
        \Dot{\bar X}_t(x) &= v(t,\bar X_t(x)), \qquad &&\bar X_{t=1}(x) =x.
    \end{aligned}
\end{equation} 
Consider the MH-MCMC in which: (i) we draw new independent sample $x^b\sim \mu_b$ at every step and use as proposal $\tilde x = X_{t=1}(x^b)$, and (ii) we accept or reject this sample using
\begin{equation}
    \label{eq:accept:reject:th}
    a(x_i,\tilde x) = \min\left( \exp\left(R(x_i,x_b)) \right) ,1\right)
\end{equation}
where 
\begin{equation}
    \label{eq:R:def}
    \begin{aligned}
        R(x_i,x_b) & = -U_*(X_{t=1}(x_b))+ U_*(x_i) + U_b(x_b) - U_b(\bar X_{t=1}(x_i)) \\
        & \quad+ \int_0^1 \left[\nabla\cdot v(t,X_t(x_b)) - \nabla\cdot v(t,\bar X_t(x_i)) \right] dt
    \end{aligned}
\end{equation}
Then, the invariant measure of this scheme is $\mu$ and~\eqref{eq:time:aveg} holds. 
\end{proposition}

\begin{proof}
    For the MH-MCMC scheme in which $\tilde x = X_{t=1}(x^b)$ with $x^b\sim \mu_b$, from~\eqref{eq:accept:reject} we have \begin{equation}
    \label{eq:accept:reject:proof}
    a(x_i,\tilde x) = \min\left( \frac{\rho_*(X_{t=1}(x^b)) \rho_{t=1}(x_i) }{\rho_*(x_i) \rho_{t=1}(X_{t=1}(x^b))},1\right)
\end{equation}
where $\rho_t$ solves the continuity equation~\eqref{eq:continuity}. Using \eqref{eq:rhot:rho0} we know that
\begin{equation}
    \label{eq:rhot:rho0:mh:1}
    \rho_{t=1}(X_{t=1}(x^b)) = \rho_b(x^b) \exp\left( -\int_0^1 \nabla \cdot v(t,X_t(x^b)) dt\right)
\end{equation}
We can also use the fact that the map $X_{t=1}$ and $\bar X_{t=0}$ are inverse of each other, i.e. $X_{t=1}(\bar X_{t=0}(x))$, and  satisfy the composition rule $X_t(\bar X_{t=0} (x)) = \bar X_t(x)$, to deduce
\begin{equation}
    \label{eq:rhot:rho0:mh:2}
    \begin{aligned}
        \rho_{t=1}(x_i) &= \rho_b(\bar X_{t=0}(x_i)) \exp\left( -\int_0^1 \nabla \cdot v(t,X_t(\bar X_{t=0}(x_i))) dt\right)\\
         &= \rho_b(\bar X_{t=0}(x_i)) \exp\left( -\int_0^1 \nabla \cdot v(t,\bar X_{t}(x_i)) dt\right)
    \end{aligned}
\end{equation}
Inserting \eqref{eq:rhot:rho0:mh:1} and \eqref{eq:rhot:rho0:mh:2} in~\eqref{eq:accept:reject} and using the explicit forms $\rho_b = Z_b^{-1} e^{-U_b}$ and $\rho_* = Z_*^{-1} e^{-U_*}$ gives~\eqref{eq:accept:reject:th}, which shows that $\mu_*$ is an invariant measure for the Markov chain. Since $\rho_{t=1}>0$ by assumption, this chain is irreducible, and~\eqref{eq:time:aveg} holds. 

\end{proof}

\begin{remark}
    When the current state of the MCMC chain $x_i$ was also constructed via transport and therefore was produced by pushing an $x_{i}^b \sim \mu_b$ through to $x_i = X_{t=1}(x_{i}^b)$, the likelihoods $\rho_*(x_i)$, $\rho_b(x_{i}^b)$, and $\rho_{t=1}(x_i)$ are already available and need not be recomputed to estimate the factor in~\eqref{eq:accept:reject:th}. That is,  the integration of the reverse map $\bar X_t$ can be avoided. However, in more advanced schemes, where local updates from alternative MCMC proposals are intertwined with independence-Metropolis steps \cite{hackett2021flowbased, samsonov2022localglobal}, or when learning with MH-MCMC, such reverse dynamics remain necessary. 
\end{remark}

\subsection{Learning with MH-MCMC}
\label{sec:leran:mcmc}

Proposition~\ref{prop:expect:mcmc:0} offers us a way to use MH-MCMC in concert with the velocity field that we are optimizing to get the data needed for this optimization. The estimators are similar to the ones given in Sec.~\ref{sec:learning:is}. For example Algorithm~\ref{alg:3} goes back to a form closer to Algorithm~\ref{alg:2}

\begin{algorithma}
\label{alg:4}
Start with an initial guess $\theta_0$ for the parameters and pick some $n\in \N$, then for $k\ge0$: 

\noindent
(i) Generate $\{x_i\}_{i=1}^n$ using the MH-MCMC scheme in Proposition~\ref{prop:expect:mcmc:0} using the velocity field $v^{\theta_{k}}$ and estimate the gradient of the loss via
\begin{equation}
    \label{eq:grad:para:obj:mh:a}
    \partial_\theta L_n(\theta_k) = \frac1n\sum_{i=1}^n  \int_0^1  \left(\partial_\theta  \nabla \cdot v^{\theta_k}(t,\bar X^i_t) -  \partial_\theta v^{\theta_k}(t,\bar X^i_t )\cdot \bar G^i_t  \right)  dt
\end{equation}
where  $(\bar X^i_t,\bar G^i_t)$ solve
\begin{equation}
    \label{eq:EL:2:d:mh:a}
    \begin{aligned}
    & \Dot {\bar X}^i_t= v^{\theta_k}(t,\bar X^i_t), \quad &&\bar X^i_{t=1} = x_i,\\
    & \Dot {\bar G}^i_t = - [\nabla v^{\theta_k}(t,\bar X^i_t)]^T \bar G^i_t + \nabla \nabla \cdot v^{\theta_k}(t,\bar X^i_t) , \quad &&\bar G_{t=0}^i = -\nabla U_b(\bar X^i_{t=0}).
    \end{aligned}
\end{equation}

\noindent
(ii) Use this estimate of the gradient to perform a step of SGD and update the parameters via
\begin{equation}
    \label{eq:sgd:is:2}
    \theta_{k+1} = \theta_k - h_k \partial_\theta L_n(\theta_k) 
\end{equation}
where $h_k>0$ is some learning rate.

\end{algorithma}

\noindent
By the same logic presented below Algorithm \ref{alg:3}, the use of backward dynamics $\bar X_t$ are not strictly necessary, given the equivalence $\bar X_t^i = X_t^i$. Note that the number $n$ of samples (i.e. the number of MH-MCMC steps performed) can be adapted at every iteration. Note also that several chain can be run in parallel, and we can to the mix data coming from any other MH-MCMC scheme whose invariant measure is $\mu_*$. 

We can proceed similarly for the loss~\eqref{eq:loss:sbdm} used in SBDM and \eqref{eq:loss:interp} used with stochastic interpolants: it simply amounts to replacing the data set $\{x_i^*\}_{i=1}^n$ with $\{x_i\}_{i=1}^n$ generated using the MH-MCMC scheme in Proposition~\ref{prop:expect:mcmc:0} using the velocity field $v^{\theta_{k-1}}$ from the previous SGD iteration (and possibly augmented with data from another MH-MCMC scheme).

\section{Prospects,  Limitations, and Generalizations}
\label{sec:conclu}

We conclude this paper with a few remarks:

\medskip

\paragraph{Nature of the map.} We have discussed ways of constructing maps between distributions under a continuous-time framework. It is worth contextualizing this in variants of the flow-based generative modeling literature. The map $X_{t=1}$ presented here is solution of the ODE~\eqref{eq:flow:map}: this is convenient as it guarantees that the map is invertible, with a Jacobian that can be easily computed. It also grounds the optimization procedure within the neural ODE framework. However, it is not necessary to proceed this way: in particular, we could replace $X_{t=1}$ by any  invertible map $T$, and proceed similarly as long as we can compute the Jacobian of this map: in essence, this is what was proposed already in~\cite{tabak2010, tabak2013} and is also at the core of real-valued non-volume preserving (real NVP) transformations~\cite{dinh2017density} and Non-linear Independent Component Estimation (NICE)~\cite{dinh2014}. These methods construct an invertible map $T$ composed of simple parametric transformations that make the Jacobian highly structured and cheap to compute, making training with a divergence such as the KL efficient at the cost of limited expressivity. This option, however, is only available if we learn the map using the direct or reversed KL divergence, since a continuous-time $X_t$ is at the core of SBDM as well as the stochastic interpolant framework.

\paragraph{Padding the target distribution to eliminate the Jacobian.} It is possible to modify the target distribution so that  dynamics with simpler characteristics can be used to construct the map. For example, we could consider the product measure $d\mu_p(x,y) = d\mu_*(x) d\mu_*(y)$ and use
\begin{equation}
    \label{eq:padded:1}
    \dot X_t = v(t,Y_t), \qquad \dot Y_t = w(t,X_t),
\end{equation}
for some velocity fields $v$ and $w$ to be learned. This formulation has the advantage that it eliminates the Jacobian term since in the extended space we have $\nabla_x v(t,y) + \nabla_y w(t,x) =0$. Another way to achieve this aim is to take $d\mu_p(x,y) = d\mu_*(x) (2\pi m)^{-d/2}e^{-|y|^2/(2m)} dy$ and use as dynamics 
\begin{equation}
    \label{eq:padded:2}
    \dot X_t = m^{-1} Y_t, \qquad \dot Y_t = w(t,X_t).
\end{equation}
This option, however, is again only available if we learn the map using the direct or reversed KL divergence: in the SBDM and the stochastic interpolant framework, the velocity is fixed by the connection between $\mu_b$ and $\mu_*$ that we start with, and it fixes the velocity field $v$. Thsi is the price to pay for using a simulation-free scheme. 

\medskip

\paragraph{Exploration capabilities.} The variational formulation based on the reverse KL divergence in Proposition~\ref{prop:var:2} is appealing because the objective can estimated empirically using samples from the base distribution $\mu_b$: this is also apparent in Algorithm~\ref{alg:1} which shows that the optimization can be performed solely using these samples. However this feature also implies that the capability for exploration of the procedure will be limited: more precisely, if we do not start with an informed guess for the velocity field $v$ (e.g. if we take it to be identically zero initially), it will in general be hard to perform the optimization and train $v$ so that the pushforward of the base distribution by $X_{t=1}$ be close to the target. This can be pathological when the target is multi-modal. In addition, using the reverse KL divergence does not easily allow us to incorporate in the learning any sample for $\mu$ we may have (either from the start or after some MCMC computation). For this reason it may be preferable to use the direct KL divergence as a loss, or use SBDM or the stochastic interpolant framework, since they all allow to incorporate new data from $\mu$ in the learning.

\bigskip
\paragraph{Funding information.}
M.S.A is supported in part by the National Science Foundation under the award PHY-2141336. 
E.~V.-E. is supported in part by the National Science Foundation under awards DMR-1420073, DMS-2012510, and DMS-2134216, by the Simons Collaboration on Wave Turbulence, grant No. 617006, and by a Vannevar Bush Faculty Fellowship. 

\bigskip
\paragraph{Acknowledgements.} We would like to thankfully acknowledge Umberto Tomasini and Yijun Wan for their notes taken during the Les Houches 2022 Summer School on Statistical Physics and Machine Learning which were useful to assemble these lectures. We are also grateful to Marylou Gabri\'e and Grant Rotskoff for discussions about map-assisted MCMC methods, to Nick Boffi and Mark Bergman for discussion about generative models based on diffusion and stochastic interpolants, and to Jonathan Niles-Weed for discussions about measure transportation theory. Finally we thank Florent Krzakala and Lenka Zdeborova for organizing the summer school in Les Houches.

\bibliography{main}

\end{document}